\documentclass[11pt]{article}

\usepackage[preprint]{acl}
\usepackage[ruled,linesnumbered]{algorithm2e} % 在导言区引用
\usepackage{amsmath}
\usepackage{amssymb}
\usepackage{amsthm}
\usepackage{caption} % 必须引用，用于在同一个环境里区分 table 和 figure 标题
\usepackage{subcaption}
\usepackage{times}
\usepackage{latexsym}
\newtheorem{proposition}{Proposition} % 定义“命题”环境
\usepackage{pgfplots}    % 定义柱状图环境
\usepackage[table]{xcolor}
\usepackage[most]{tcolorbox}
\newtcolorbox{PromptBox}[1]{
    breakable,
    colback=gray!5!white, colframe=black, sharp corners, boxrule=0.8pt,
    left=8pt, right=8pt, fontupper=\small, enhanced,
    attach boxed title to top left={xshift=-0.8pt, yshift=-0.8pt},
    boxed title style={sharp corners, colback=black},
    title=\textbf{#1}
}

\usepackage{booktabs}
\usepackage{multirow}
\usepackage[table]{xcolor}
\usepackage{colortbl}
\usepackage{pifont}

\definecolor{lightgray}{gray}{0.95}
\definecolor{gain}{HTML}{008000}
\definecolor{avgblue}{HTML}{E8F0FE} % 浅蓝色背景

\usepackage[T1]{fontenc}
\usepackage[utf8]{inputenc}

\usepackage{microtype}

\usepackage{inconsolata}

\usepackage{graphicx}

\title{Stabilizing Unsupervised Self-Evolution of MLLMs via Continuous Softened Retracing reSampling}

\author{
  \textbf{Yunyao Yu}\textsuperscript{1,$^*$},
  \textbf{Zhengxian Wu}\textsuperscript{1,$^*$},
  \textbf{Zhuohong Chen}\textsuperscript{1,$^*$},
  \textbf{Hangrui Xu}\textsuperscript{2},
  \textbf{Zirui Liao}\textsuperscript{1}, \\ 
  \textbf{Xiangwen Deng}\textsuperscript{3},
  \textbf{Zhifang Liu}\textsuperscript{1},
  \textbf{Senyuan Shi}\textsuperscript{4},
  \textbf{Haoqian Wang}\textsuperscript{1,$\dagger$} \\
  \textsuperscript{1}Tsinghua University, \textsuperscript{2}Hefei University of Technology, \\
  \textsuperscript{3}University of Arizona,
  \textsuperscript{4}MAIS, Institute of Automation, Chinese Academy of Sciences \\
}

\begin{document}
\maketitle
\begin{abstract}
In the unsupervised self-evolution of Multimodal Large Language Models, the quality of feedback signals during post-training is pivotal for stable and effective learning. However, existing self-evolution methods predominantly rely on majority voting to select the most frequent output as the pseudo-golden answer, which may stem from the model's intrinsic biases rather than guaranteeing the objective correctness of the reasoning paths. To counteract the degradation, we propose \textbf{C}ontinuous \textbf{S}oftened \textbf{R}etracing re\textbf{S}ampling (\textbf{CSRS}) in MLLM self-evolution. Specifically, we introduce a Retracing Re-inference Mechanism (\textbf{RRM}) that the model re-inferences from  anchor points to expand the exploration of long-tail reasoning paths. Simultaneously, we propose Softened Frequency Reward (\textbf{SFR}), which replaces binary rewards with continuous signals, calibrating reward  based on the answers' frequency across sampled reasoning sets. Furthermore, incorporated with Visual Semantic Perturbation (\textbf{VSP}), CSRS ensures the model prioritizes mathematical logic over visual superficiality. Experimental results demonstrate that CSRS significantly enhances the reasoning performance of Qwen2.5-VL-7B on benchmarks such as MathVision. We achieve state-of-the-art (SOTA) results in unsupervised self-evolution on geometric tasks. Our code is avaible at \url{https://github.com/yyy195/CSRS}.
\end{abstract}

\section{Introduction}

Self-evolution \citep{wu_meta-rewarding_2024,chen_mathse_2025,jiang_bootstrapping_2025,DBLP:journals/corr/abs-2511-16672,wu2026modelsjudgethemselvesunsupervised} in Multimodal Large Language Models (MLLMs) is an emerging paradigm that leverages the model's intrinsic reasoning capabilities to achieve continuous improvement through internal feedback. To alleviate the dependency on costly manual annotations, this unsupervised framework demonstrates immense potential by bypassing the labeling bottleneck and it enables autonomous  evolution of the model's capablity through leveraging unlabeled data across tasks such as geometric reasoning.

Recent works have focused on generating self-improving feedback signals. For instance, MM-UPT \citep{wei_first_2025} establishes an unsupervised post-training pipeline, utilizing a majority voting self-rewarding mechanism based on the GRPO \citep{shao2024deepseekmath} algorithm. Similarly, VisPlay \citep{he2025visplay} enhances visual reasoning capabilities and reduces hallucinations by decoupling the base  MLLM into two roles: a conditional questioner and a multimodal reasoner. 

However, rewards of these approaches primarily generate from the model's internal preferences, leading to a confirmation bias cycle: the model repeatedly reinforces  paths it perceives as correct. This narrows the exploration horizon, eventually resulting in the loss of logical diversity and distributional shift often called "Model Collapse" \citep{shumailov2023curse,shafayat_can_2025,swamy2025all}. Furthermore, these methods rely heavily on majority voting, which provides only coarse-grained and sparse reward signals derived entirely from the model’s initial biases, thereby exacerbating the model collapse.

\begin{figure*}[t]
  \centering
  \includegraphics[width=1.0\linewidth]{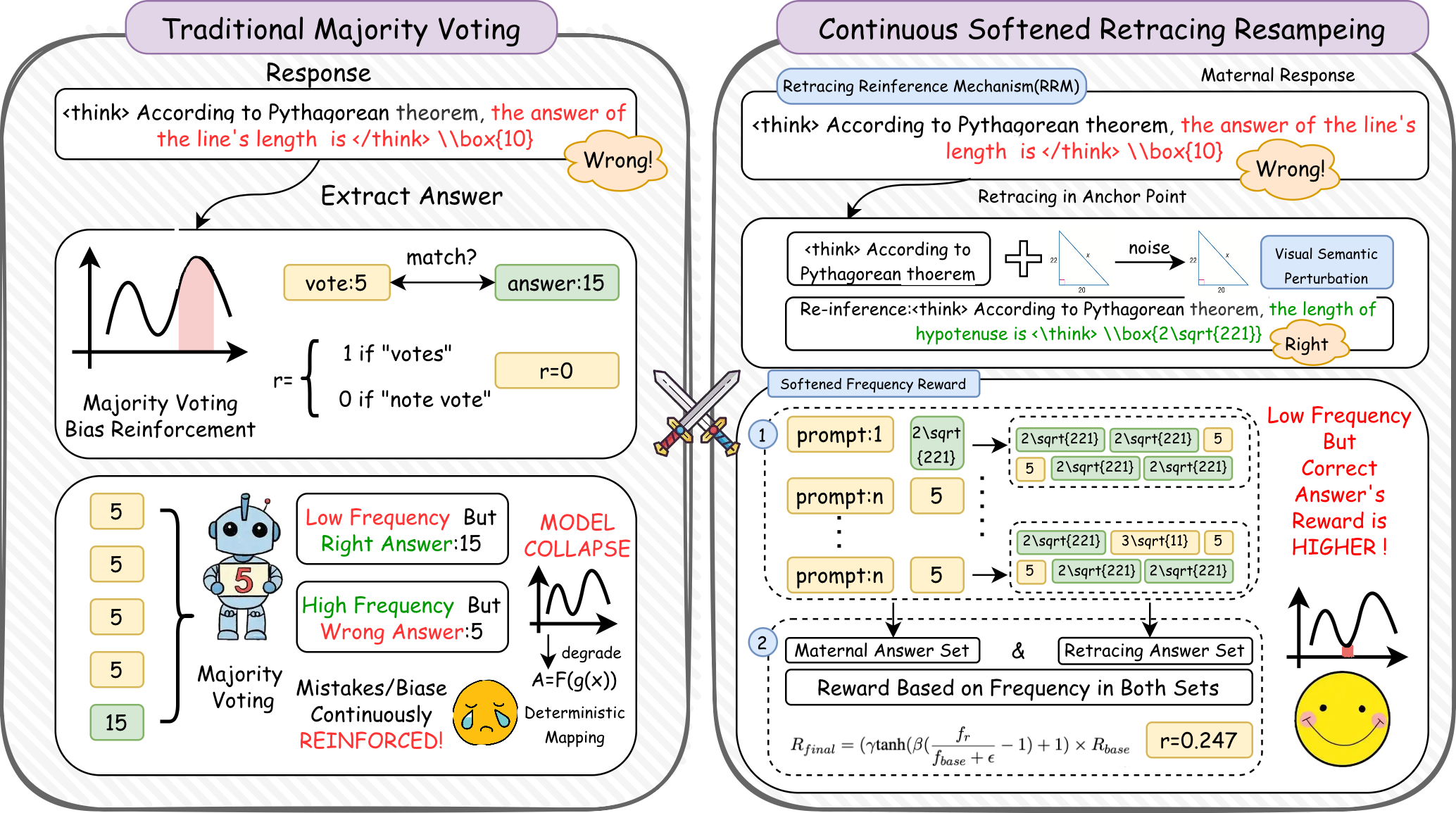} % 调整 width 来控制图片大小
  \caption{\textbf{Overview of our work.} During unsupervised self-evolutionary reinforcement learning, traditional majority voting methods (left) for pseudo-labeling rely solely on the model's inherent biases. This frequently leads to model collapse, where the model degenerates into a deterministic mapping and fails to explore the true solutions. To alleviate this phenomenon, our method introduces CSRS (right) to reduce the occurence of situations.}
  \label{pipeline}
\end{figure*}

To overcome these challenges on geometric reasoning tasks, we propose \textbf{C}ontinuous \textbf{S}oftened \textbf{R}etracing re\textbf{S}ampling (\textbf{CSRS}). As shown in Fig. \ref{pipeline}, by establishing retracing anchor points within the maternal trajectory -- initial responses and integrating softened frequency rewards with visual perturbation, CSRS effectively mitigates model collapse in unsupervised settings. Specifically, we introduce Retracing Re-inference Mechanism (\textbf{RRM}), which sets a retracing anchor point within the maternal response answers and restarts inference from anchor points to construct a re-inference answer set. Unlike static sampling in conventional methods, this mechanism compels the model to perform deep exploration at key logical decision nodes. To address the collapse triggered by majority voting, we introduce Softened Frequency Reward (\textbf{SFR}). Rather than assigning discrete binary scores, this mechanism calculates base rewards based on the occurrence frequency of the problem's answer in the union of the re-inference and maternal sets. By utilizing the frequency variance between two sets to calibrate the reward score, we incentivize the model to explore low-probability but correct responses. Additionally, we introduce Visual Semantic Perturbation (\textbf{VSP}) during the re-inference phase, forcing the model to further prioritize invariant mathematical logic over superficially low-level image features.

We conducted a comprehensive evaluation of CSRS across four major multimodal geometric reasoning benchmarks based on Qwen2.5-VL-7B \citep{bai2025qwen2}: MathVision \citep{wang2024measuring}, MathVista \citep{lu2023mathvista} MathVerse \citep{zhang2024mathverse} and We-Math \citep{qiao2025we}. Notably, compared to the baseline (MM-UPT), we gains 2.34\%, 1.48\%, 2.41\%, 3.28\% respectively on benchmarks. Experimental results show that CSRS achieves state-of-the-art (SOTA) performance in the unsupervised self-evolution.  The primary contributions of this work are as follows:

(1) We propose the CSRS, which uses retracing re-inference anchor points combined with softened frequency rewards to theoretically and empirically mitigate the model collapse inherent in majority voting mechanism.

(2) Moving beyond binary rewards, CSRS dynamically amplifies the rewards for low-frequency but potentially correct long-tail reasoning paths, significantly improving the robustness of MLLM self-evolution's inference capability.

(3) CSRS achieves leading performance across four multimodal mathematical benchmarks. In particular, on MathVista, it reaches an accuracy of 68.25\%, demonstrating the substantial potential of unsupervised self-evolution without the need for expensive human intervention.

\section{Related Work}

Reinforcement Learning (RL) \citep{xu2025single,yu2025gpo,lin2025understanding,chen2026r3greasoningretrievalrerankingframeworkvisioncentric} has emerged as a pivotal tool for unlocking the latent reasoning capabilities of Large Language Models (LLMs). These advancements span from established algorithms like PPO \citep{schulman2017proximal} and DPO \citep{rafailov2023direct} to the recently introduced GRPO \citep{shao2024deepseekmath}. However, these methods primarily operate within the paradigm of supervised reinforcement learning, where reliable parameter updates relies heavily on high-quality, external ground-truth labels. This strict dependency on external signals inherently constrains the model’s capacity for autonomous evolution in scenarios where human annotations are expensive or unavailable.

To bridge this gap, the unsupervised self-evolution paradigm has gained significant traction as it enables models to autonomously generate pseudo-labels to guide their own iterative optimization \citep{liu_spice_2025,kwan_opensir_2025}. Existing works attempts to substitute human labels with signals from diverse sources.  AbsoluteZero \citep{zhao_absolute_2025} uses external verifiers, such as math or code executors (e.g., MathVerify) to validate reasoning outcomes. While this reduces manual labeling effort, the feedback remains constrained to the coverage of external rules, essentially remaining a form of supervised training \citep{wang_enhancing_2025,dong2025agentic}. Other researchs seeks complete independence from external intervention. MM-UPT \citep{wei_first_2025} introduces a framework that adopts majority voting results among responses as pseudo-labels. EvolRL \citep{zhou2025evolving} prioritizes the stability of majority voting while refining rewards with novelty scores to enhance exploratory diversity. MultiAgentEvolve \citep{chen_multi-agent_2025} and VisPlay \citep{he2025visplay} use self-play dynamics by decoupling the model into two roles: a challenging questioner and an adaptive responder.

Despite these explorations into unsupervised self-evolution, most existing approaches are restricted to providing binary (0/1) hard rewards based solely on the final answer \citep{kwan_opensir_2025,liu_spice_2025}. Such coarse-grained feedback is insufficient for guiding the model through critical reasoning points. Worse still, these signals exacerbate a self-reinforcement trap, where the model continuously amplifies its initial distributional biases, accelerating the onset of model collapse. This degradation in signal quality ultimately leads to  model collapse, causing training trajectories to deviate from valid logical reasoning.

\section{Methodology}

This section is structured as follows: First, Section 3.1 provides a theoretical analysis of the underlying causes of model collapse in MLLM self-evolution under conventional majority voting, followed by how to mitigate this issue through softened frequency rewards. Subsequently, Section 3.2 introduces the CSRS framework, detailing its technical components including retracing re-inference, softened frequency rewards and visual semantic perturbation strategies. 

\subsection{Preliminaries}
\label{math_proof}

\subsubsection{Continuous Rewards as a Stabilizer Compared to 0-1 Rewards}

This study is grounded in the Group Relative Policy Optimization (GRPO) algorithm. In an unsupervised self-evolution scenario, given a reference distribution $P_{ref}$, the model updates its policy $P$ by maximizing the relative advantage of intra-group samples $A(x) = r(x) - \bar{r}$ subject to a KL divergence constraint. The objective function is: \begin{equation}
    \max_{P} \mathbb{E}_{x \sim P}[A(x)] - \beta \mathbb{D}_{KL}(P \| P_{ref})
\end{equation} By constructing the Lagrangian functional:
\vspace{-4pt}
\begin{align}
L(P) = & \sum_x P(x)A(x) - \beta \sum_x P(x) \log \frac{P(x)}{P_{ref}(x)} \nonumber \\
& + \lambda \left( \sum_x P(x) - 1 \right) \label{eq:lagrangian}
\end{align}
and taking the derivative with respect to $P(x)$, we derive the ideal closed-form solution for policy iteration during self-evolution:
\vspace{-4pt}
\begin{equation}
P_{n+1}(x) = \frac{P_n(x) e^{A_n(x)/\beta}}{Z}
\end{equation}

where $Z$ is the normalization constant, $\beta$ is KL regularization weight and $n$ denotes the number of self-evolution iteration. This derivation reveals that the essence of unsupervised self-evolution is an \textbf{exponential self-reinforcement process}, where the evolutionary trajectory of the distribution is driven by the advantage function $A_n(x)$. 

However, the conventional Majority Voting (MV) mechanism defines $r_n(x)$ as a binary reward at the $n$-th evolution iteration, a discrete signal that triggers over-concentration of gradient updates on high frequency paths. Marginal biases in the initial distribution are rapidly polarized through iteration, causing the probability density to map extremely toward high-frequency mode regions. In the limit, the model degenerates into a deterministic mapper \citep{yun_alignment_2025}, stripped of its exploratory capacity. To quantify and mitigate this process, we define the Contrastive Factor $G_n$ denoted as: \begin{equation}
     G_n = \frac{R_{n+1}}{R_n} = \exp(\eta \Delta r_n)
\end{equation}  where $\eta$ is $\displaystyle \frac{1}{\beta}$ and  $R_n = \frac{P_n(x_1)}{P_n(x_2)}$. $x_1$ are majority samples and $x_2$ are long-tail samples. Under the MV mechanism, $\Delta r_{MV} = 1$ causes the contrastive factor to expand exponentially at the maximum rate of $G_{MV} = e^{\eta}$. 

In contrast, our proposed softened frequency reward provides a damping effect by introducing continuity into the probability density. Let $\rho$ denote the proportion of mode answers and $\epsilon$ denote the proportion of long-tail answers. Then:
{\setlength{\abovedisplayskip}{4pt}   % 公式上方距离
\setlength{\belowdisplayskip}{4pt}   % 公式下方距离
\setlength{\abovedisplayshortskip}{0pt}
\setlength{\belowdisplayshortskip}{0pt}
\begin{equation}
G_{SR} = \exp(\eta(\rho - \epsilon)) < G_{MV} = \exp(\eta)
\end{equation}
Since $0 < \rho - \epsilon < 1$, this inequality proves that the SFR mechanism effectively curtails the trend of distributional polarization by reducing the contrastive gain, thereby preserving the necessary logical diversity for the model during the self-evolution process. Full mathematical proofs for the propositions discussed in this section are deferred to Appendix \ref{Proof_of_pro} in detail.

\begin{algorithm}[t] % 去掉星号，改为单栏；[htbp] 增加放置灵活性
\caption{Continuous Softened Retracing reSampling(CSRS)}
\label{alg:csrs}
\SetKwInOut{Require}{Require}

\Require{Dataset $\mathcal{D}$, model $\theta$, retracing rate $\omega$, resample $m$, learning rate $\alpha$, maternal num $n$.}

Initialize optimizer $\mathcal{O}$ with $\theta$\;

\For{each minibatch $\{x\}$ in $\mathcal{D}$}{
    % 采样 maternal 答案和轨迹
    Sample $\mathcal{A}_n, \tau \sim \pi_\theta( \mid x)$\;

    \For{$t = 1$ to $n$}{
        % 缩短公式或将其拆分，防止单栏溢出
        $\tau' \leftarrow \text{Concat}(x_i^p, \text{retrace}(\tau,\omega))$\;
        Visual Perturbation $x_i^i \leftarrow x + \epsilon$\;
        Resample new $\mathcal{A}_m$ from $\tau', x_i^i$\;
    }
    $\mathcal{A}_{all} \leftarrow \mathcal{A}_n \cup \mathcal{A}_m$\;
    \For{$i = 1$ to $n$}{
        $r_{base}^i \leftarrow \mathcal{A}_n(a_i)/\mathcal{A}_{all}$\;
        $r_{fre} \leftarrow \text{SFR}(\mathcal{A}_{m},\mathcal{A}_{all},r_{base}^i)$\;
        $r \leftarrow r_{fre}$\; % 修正了变量名一致性
    }
    
    $b \leftarrow \text{ComputeBaseline}(\cdot)$\;
    $g \leftarrow \nabla_\theta \log \pi_\theta(a_0 \mid x)(r - b)$\;
    $\theta \leftarrow \theta + \alpha \mathcal{O}(g)$\;
}
\Return $\theta$
\end{algorithm}

\subsubsection{Retracing Re-inference Overcomes Sampling Sparsity}

In long-range reasoning tasks, the candidate path space expands exponentially at a rate of $\mathcal{O}(b^L)$ where $L$ is the reasoning length and $b$ is the branching factor such as vocabulary size, rendering correct reasoning paths extremely rare, a phenomenon known as sampling sparsity. Under a finite computational budget, stochastic sampling struggle to cover critical logical nodes. Consequently, early-stage biases are amplified through the recursive process, exacerbating the risk of model collapse.

To address this, we introduce the concept of local stability, which posits that the robustness of a logical path should be calibrated by answer consistency at key decision points. Rather than merely narrowing the search, CSRS implements this by establishing retracing anchor points at intermediate stages to trigger local resampling. This mechanism induces a deliberate distributional divergence from the original trajectory, effectively breaking the cumulative bias of the initial policy. It expands the model's reach into the long-tail sample space, providing a vital opportunity to discover correct but low-probability paths that would otherwise be suppressed by the primary distribution. This approach effectively curtails model collapse and provides more robust gradient guidance for the MLLM self-evolution process.

\begin{figure*}[t]
  \centering
  \includegraphics[width=1.0\linewidth]{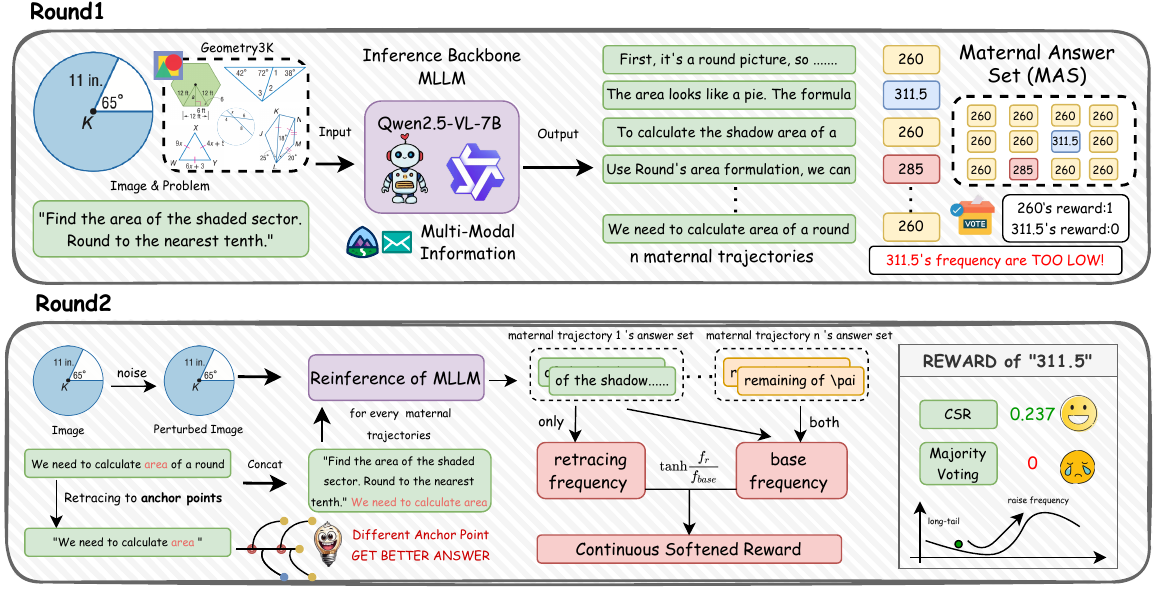} % 调整 width 来控制图片大小
  \caption{\textbf{Pipeline of our method}. (Round1) illustrate initially maternal trajectories and answers generated by MLLM. (Round2) shows three core components of CSRS and improvement to rewards.}
  \label{Method_pipeline}
\end{figure*}

\subsection{Continuously Softened Retracing reSampling}

\subsubsection{Theoretical Assumptions}

To systematically formalize our method, we establish the following fundamental assumptions about mathematical reasoning:

\textbf{Semantic Uniqueness and Consistency of Reasoning.} For a mathematical problem of a given complexity, the ideal reasoning path and its corresponding terminal result should remain semantically unique and consistent. We assume that if a model regenerates its reasoning paths starting from any intermediate step, the underlying mathematical logic and the final conclusion must remain invariant, regardless of linguistic variations or paraphrasing in the expression.

\textbf{Existence of Critical Reasoning Pivots.} We hypothesize that mathematical reasoning is not a uniform linear progression but is governed by critical reasoning pivots, retracing anchor points as human solve a problem. These pivots are key steps that shape the following reasoning and determine whether the final answer is correct.

\subsubsection{Algorithm Details}

To overcome the training instability and model collapse caused by the bias in majority voting during unsupervised self-evolution, we propose the Continuous Softened Retracing reSampling (CSRS). This framework comprises three synergistic components: Retracing Re-inference Mechanism (RRM), Softened Frequency Reward (SFR) and Visual Semantic Perturbation (VSP). The picture of our method pipeline can be seen in Fig. \ref{Method_pipeline}.

\textbf{Retracing Re-inference Mechanism (RRM).} In long-range mathematical reasoning tasks, subtle early-stage logical deviations tend to amplify exponentially as the reasoning steps progress. To capture and rectify these intermediate errors, we introduce a Retracing Re-inference Mechanism (RRM) . Let $\mathcal{D}$ denote a dataset consisting of prompt-image pairs $(X_p, X_i)$. For a specific pair sampled from $\mathcal{D}$, the text prompt $X_p$ is represented as a sequence of $k$ tokens, denoted by $X_p = (x_1, x_2, \dots, x_k)$ and the model generates $n$ initial reasoning paths, designated as maternal trajectories denoted as $\mathcal{A}_n=\{R_{p_1},\cdots,R_{p_n}\}$. Subsequently, a critical logical node termed the retracing anchor point, is identified within each trajectory. For every response trajectory $ R_p=(y_1,y_2,\cdots,y_n)$, we truncate it at this anchor point and concatenate it with the original prompt: 
\begin{equation}
X_p'=\text{Concat}(X_p,(y_1,y_2,\cdots,y_{\omega \text{len}(R_p)}))
\end{equation}

where $\omega \in (0,1)$ denotes the retracing rate, and $y_{\omega \text{len}(R_p)}$ represents the retracing anchor point located at the position $\lfloor \omega \cdot \text{len}(R_{p_n}) \rfloor$ within the initial response sequence. Then we guide the model to perform $m$ local explorations for each prompt, resulting in a total of $mn$ re-inference trajectories denoted as $\mathcal{A}_m=\bigcup_{i=1}^{n}\{R_{p_i}^1,R_{p_i}^2,\cdots,R_{p_i}^m\}$ . This design transforms a broad global search into dense local sampling, ensuring no potential correct steps are missed at key logical forks.

\textbf{Softened Frequency Reward (SFR).} To mitigate the model collapse induced by the binary  hard rewards of conventional majority voting, we design the Softened Frequency Reward (SFR). We first construct a comprehensive set $\mathcal{A}_{all} = \mathcal{A}_n \cup \mathcal{A}_m$ containing all $(m+1)n$ trajectory answers. For any answer $a$ from a maternal trajectory, its base reward $R_{base}$ is defined as its statistical frequency within the entire set:
\begin{equation}
R_{base}(a) = f_{base}(a) =  \frac{\text{Count}(a, \mathcal{A}_{all})}{(m+1)n}
\end{equation}
\vspace{-4pt}
%%%%%

\begin{table*}[t]
\centering
\caption{Comparison with supervised methods. Our unsupervised method CSRS shows competitive performance against base models and current SOTA MM-UPT across various benchmarks.}
\label{tab:comparison}
\small
\setlength{\tabcolsep}{5pt} % 稍微收缩列间距以防溢出
\begin{tabular}{c c l c c c c c >{\columncolor{avgblue}}c}
\toprule
\textbf{Model and Methods} & \textbf{Unsupervised?} & \textbf{Training Data} & \textbf{MathVision} & \textbf{MathVerse} & \textbf{MathVista} & \textbf{WeMath} &  \textbf{Avg} \\
\midrule
Qwen2.5-VL-3B & {\color{red}\ding{55}} & - & 22.47 & 34.54 & 62.30 & 57.53 & 44.21 \\

Qwen2.5-VL-7B & {\color{red}\ding{55}} & - & 25.40 & 44.24 & 66.42  & 67.65 & 50.93 \\

\midrule

MM-UPT & {\color{green}\ding{51}} & Geometry3K & 26.95 & 44.53 & 66.47  & 68.49 & 51.61 \\

MM-UPT  & {\color{green}\ding{51}} & GeoQA & 26.61 & 44.15 & 65.84 & 68.25 & 51.21\\

MM-UPT  & {\color{green}\ding{51}} & MMR1 & 25.98 & 45.12 & 66.27 & 69.14 & 51.63 \\

\midrule
\textbf{CSRS (Ours)} & {\color{green}\ding{51}} & Geometry3K & 27.97 & \textbf{46.01} & 67.81 & \textbf{71.77}  & 53.39 \\

\textbf{CSRS (Ours)}  & {\color{green}\ding{51}} & GeoQA & \textbf{28.95} & 45.82 & \textbf{68.25} & 69.32 & 53.09 \\

\textbf{CSRS (Ours)}  & {\color{green}\ding{51}} & MMR1 & 27.86 & 45.89 & 67.81 & 70.53 & 53.05 \\

\bottomrule
\end{tabular}
\end{table*}

%%%%%

\begin{table*}[t] % 星号表示横跨双栏，[t] 确保排在页面顶部
  \centering
  \small
  
  % --- 第一个子表：Ablation Studies ---
  \begin{minipage}{0.40\textwidth} 
    \centering
    \caption{Sensitivity analysis of $\gamma$ and $\beta$.}
    \label{tab:sfr_sensitivity}
    \setlength{\tabcolsep}{4pt} % 压缩间距以适应 minipage 宽度
    \small % 保持字号一致
    \begin{tabular}{cc|cc} 
      \toprule
      \boldmath$\gamma (\beta=\text{5.0})$  & \textbf{Acc. (\%)} & \boldmath$\beta (\gamma=\text{0.2})$ & \textbf{Acc. (\%)} \\ \midrule
      0.05 & 44.65 & 3.5 & 45.16 \\
      0.1  & 45.12 & 5.0 & \textbf{46.01} \\
      0.2  & \textbf{46.01} & 6.5 & 45.87 \\
      0.35 & 43.78 & 8.0 & 45.72 \\ \bottomrule
    \end{tabular}
  \end{minipage}
  \hfill % 撑开中间间距
  % --- 第二个子表：Rollout Group Size ---
  \begin{minipage}{0.28\textwidth}
    \centering
    \caption{Re-inference Rollout.}
    \label{tab:exp_rollout}
    \begin{tabular}{cc} % 数据全居中
      \toprule
      \textbf{Size ($R$)} & \textbf{MathVista} \\ \midrule
      $R = 4$  & 66.45 \\ 
      $R = 5$  & \textbf{68.25} \\
      $R = 6$  & 67.46 \\
      $R = 7$  & 66.01 \\
      $R = 8$  & 64.98 \\ \bottomrule
    \end{tabular}
  \end{minipage}
  \hfill % 撑开中间间距
  % --- 第三个子表：Retracing Rate ---
  \begin{minipage}{0.30\textwidth}
    \centering
    \caption{Retracing Rates.}
    \label{tab:exp_retracing}
    \begin{tabular}{cc} % 数据全居中
      \toprule
      \textbf{Retracing Rate} & \textbf{MathVista} \\ \midrule
      $\omega$=0.1 & 66.31 \\ 
      $\omega$=0.3 & 67.01 \\
      $\omega$=0.5 & 66.58 \\
      $\omega$=0.7 & \textbf{68.25}\\
      $\omega$=0.9 & 66.43 \\ \bottomrule
    \end{tabular}
  \end{minipage}

  \vspace{-5pt} % 调节表格与下方正文的间距
\end{table*}

%%%%%%

To further validate reasoning stability, we evaluate the frequency variance of the answer between the local re-inference set $\mathcal{A}_{m}$ and the comprehensive set $\mathcal{A}_{all}$. This variance essentially quantifies the sampling distribution divergence, enabling the model to deviate from its original biased distribution and actively explore potential long-tail samples that are often neglected during global sampling. If an answer's frequency increases after local re-inference, it validates the latent robustness of that specific reasoning path. This trend signifies that the path possesses superior conditional consistency despite being a long-tail sample in the model's global distribution,  prompting a hierarchical adjustment accordingly as follows:
\begin{equation}
    R_{final} = (\gamma \text{tanh}(\beta(\frac{f_r}{f_{base} + \epsilon}-1))+1) \times R_{base}
\end{equation} where $\displaystyle f_r = \frac{\text{Count}(a,\mathcal{A}_m)}{mn}$ denotes answer's frequency in re-inference answers set and $\gamma,\beta$ are hyperparameters. By adjusting $\beta$, the reward function becomes highly sensitive to correct but low-frequency paths in the long-tail region. Meanwhile, the saturation property of $\tanh$ caps the reward for high-frequency samples, preventing the model from over-fitting to common errors and effectively mitigating model collapse. SFR encourages exploration in these "initially low-frequency but potentially correct" regions, serveing as a proxy for process-level supervision.

\textbf{Visual Semantic Perturbation.} To prevent the model from cheating by exploiting superficial visual features rather than deep mathematical logic, we introduce Visual Semantic Perturbation (VSP) during the re-inference phase. By applying Gaussian noise $\epsilon \sim \mathcal{N}(0, \sigma^2)$ to the original image $I$:
\begin{equation}
I' = I + \epsilon
\end{equation}
This perturbation forces the model to rely on invariant logical structures, ensuring that the self-evolution process is driven by genuine reasoning rather than visual heuristics in images.

\begin{figure*}[t]
  \centering
  % --- 第一个子图 ---
  \begin{subfigure}{0.48\linewidth}
    \centering
    \includegraphics[width=\linewidth]{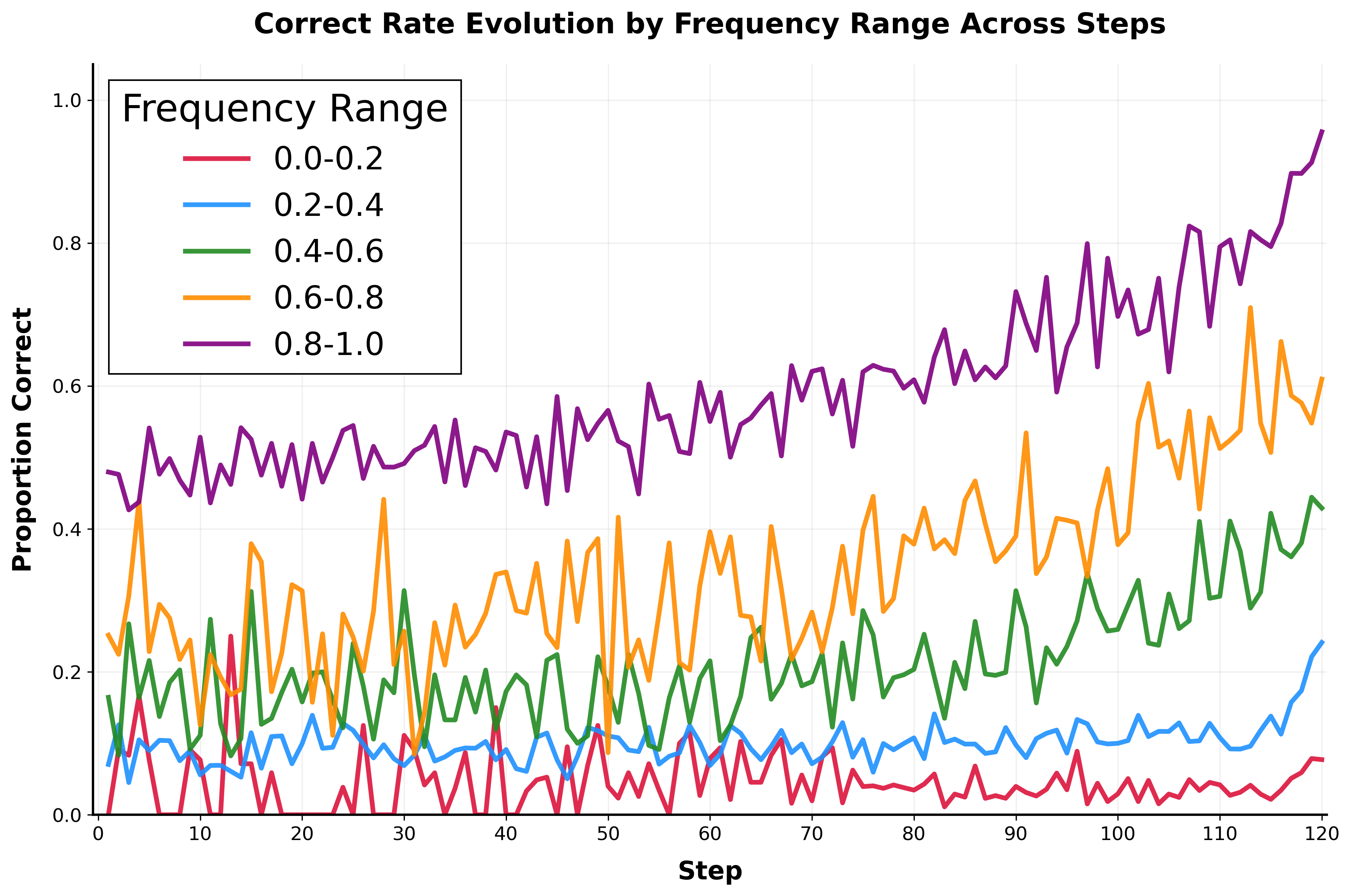}
    \caption{Accuracy of answers within each frequency interval.} % 子图 a 的标题
    \label{fig:sub_correct}
  \end{subfigure}
  \hfill
  % --- 第二个子图 ---
  \begin{subfigure}{0.48\linewidth}
    \centering
    \includegraphics[width=\linewidth]{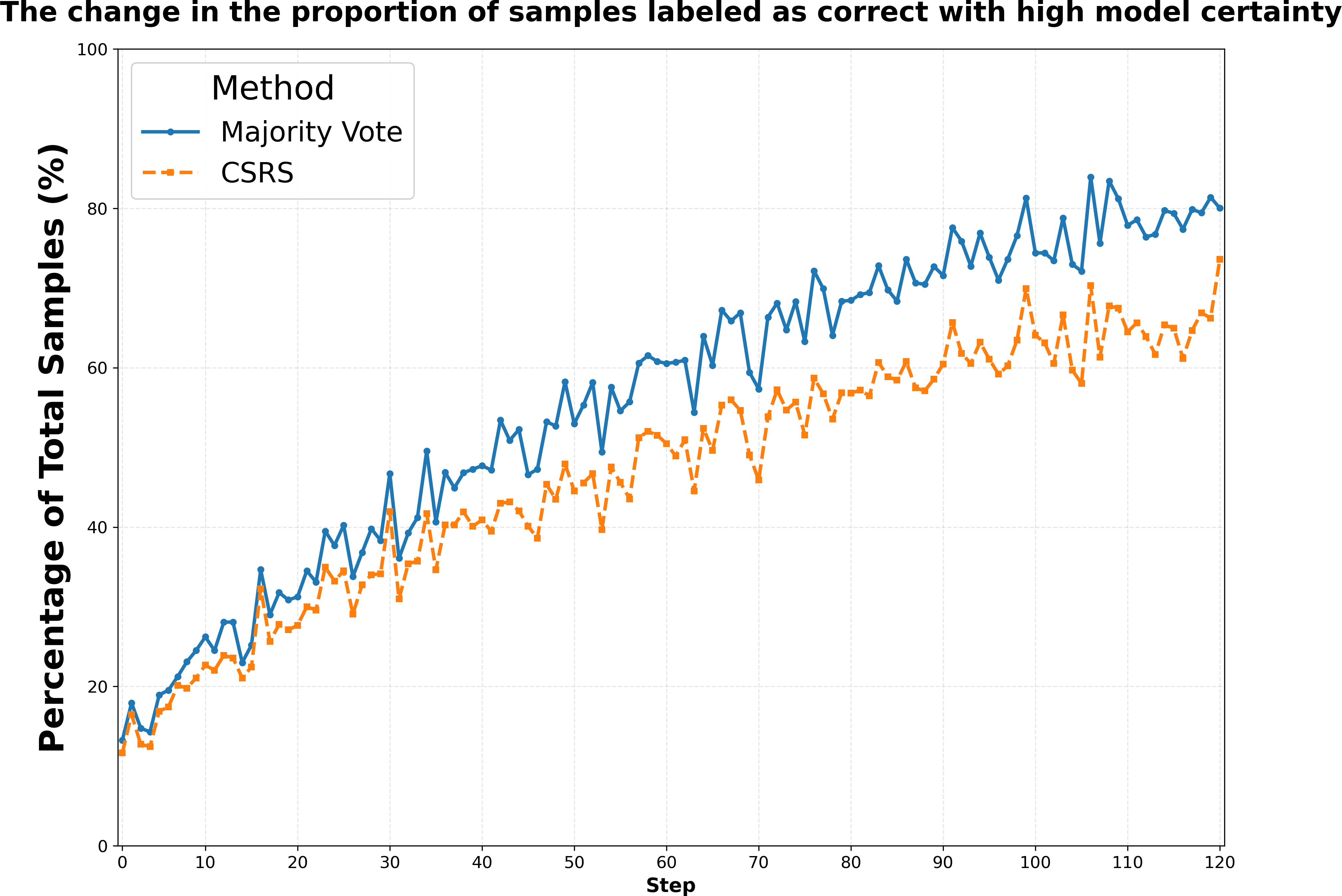}
    \caption{High certainty samples distribution percentage.} % 子图 b 的标题
    \label{fig:sub_dist}
  \end{subfigure}
  
  \caption{\textbf{Visualization of the accuracy and  propotion of high-confidence samples during training}. (a)  Answer accuracy calculated within partitioned frequency intervals, where each sample is uniquely assigned based on its frequency. (b) The evolution of high-confidence (frequence $\in$ $[0.8-1.0]$) samples proportions during training, which serves as a key indicator of model collapse.}
  \label{fig:total_figure}
  
  % 局部微调：缩短与下方正文的距离
  \vspace{-2pt} 
\end{figure*}

\begin{figure*}[t]
  \centering
  % 直接放置图片，设置宽度为页面总宽度的比例（如 0.8 或 0.9）
  \includegraphics[width=0.98\textwidth]{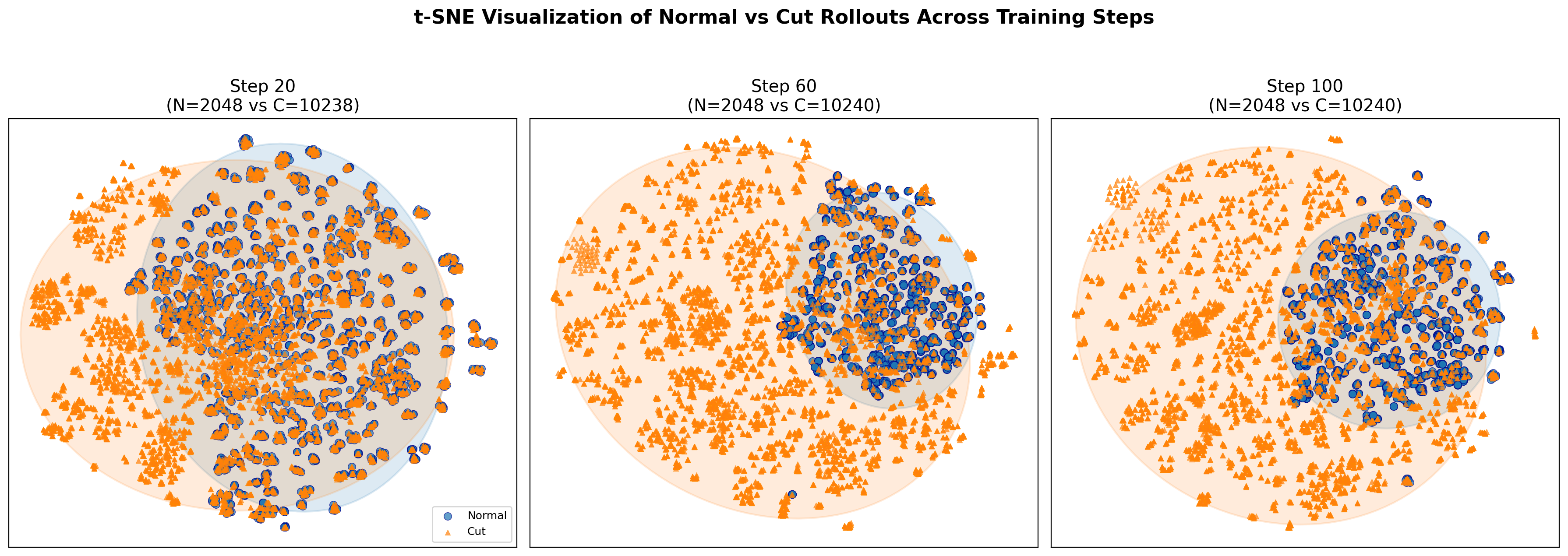}
  
  \caption{\textbf{The change of distribution of Maternal Answers Set and Re-inference Answers Set}. The figure shows different distribution of samples in two sets at the step of 20,60,100. \textbf{Blue dots} represent the semantic space of vanilla responses, while \textbf{orange dots} denote the semantic space of responses generated by CSRS.}
  \label{fig:tsne}
  % 局部微调：如果觉得图片和下方正文太近，可以保留或调整这个值
\end{figure*}

\section{Experiment}

In this section, we conduct extensive experiments to evaluate the performance of CSRS. We begin by detailing the experimental configuration in Section 4.1, covering the benchmarks, evaluation metrics, and specific implementation details. Subsequently, Section 4.2 presents our primary results along with a comprehensive performance analysis across various multimodal reasoning tasks. Finally, we perform ablation and visualization studies in Section 4.3 to evaluate the effectiveness of each component in our method. Specific cases study are clearly presented in Appendix \ref{case_study}.

\subsection{Experimental Setup}

\textbf{Datasets and Benchmarks.} In this study, we focus on the stability of internal signals in unsupervised self-evolution.  We select Geometry3K \citep{lu2021inter}, GeoQA \citep{chen2021geoqa}, MMR1 \citep{leng2025mmr1} as the primary training sets for the self-evolution process. To evaluate the reasoning capabilities of our model, we use four mainstream multimodal scientific reasoning benchmarks: (1) MathVision \citep{lu2021inter} , (2) MathVerse \citep{zhang2024mathverse} , (3) MathVista \citep{lu2023mathvista} , and (4) We-Math \citep{qiao2025we} . These benchmarks provide a comprehensive evaluation ranging from formal geometric problems to complex mathematical reasoning situated in diverse real-world contexts.

\noindent \textbf{Implementation Details.} Our unsupervised post-training pipeline is implemented using the veRL \citep{DBLP:journals/corr/abs-2409-19256} framework, built upon the GRPO algorithm. Specifically, the training is conducted over 15 epoches. We use the AdamW\citep{loshchilov2017decoupled} optimizer with an initial learning rate of $1 \times 10^{-6}$ and a weight decay of 0.01. The KL divergence constraint $\beta$ in GRPO is set to 0.01 to regularize policy updates. We initially generate $n$=8 maternal trajectories for each prompt, followed by $m$=5 local re-inference trajectories starting from the retracing anchor points. Other configurations are detailed in Appendix \ref{config}.

\subsection{Experiment Results}

In this section, we evaluate  \textbf{C}ontinuous \textbf{S}oftened \textbf{R}etracing re\textbf{S}ampling (\textbf{CSRS})} by comparing it with the state-of-the-art MM-UPT \citep{wei_first_2025} baseline and the original Qwen2.5-VL \citep{bai2025qwen2} base models. To ensure a fair and rigorous comparison, both CSRS and MM-UPT are post-trained using the same Qwen2.5-VL-7B  backbone under identical experimental settings. The experimental results on four benchmarks MathVision, MathVerse, MathVista and We-Math are summarized in Table~\ref{tab:comparison}. 

As demonstrated in our results, CSRS achieves consistent performance gains over both the  base models and the MM-UPT baseline. Specifically, when utilizing Qwen2.5-VL-7B as the backbone, the CSRS achieves scores of \textbf{28.95\%}, \textbf{46.01\%},  \textbf{68.25\%}, \textbf{71.77\%} respectively on four benchmarks, outperforming the MM-UPT baseline by \textbf{2.34\%, 1.48\%, 2.41\%, 3.28\%} respectively and  surpassing the performance of the original Qwen2.5-VL-7B . These results suggest that CSRS is effective in complex geometric reasoning tasks, mitigating the model collapse while better exploiting high-quality reasoning paths.

\subsection{Ablation Experiments and Visualization}

\textbf{The Synergistic Effect of RRM, SFR and VSP.}
 As demonstrated in Table \ref{tab:ablation_column_width}, removing either module results in a substantial performance degradation. This indicates that the model can overcome training instability in unsupervised learning when retracing provides diverse exploratory samples in long-tail regions. Furthermore, as illustrated in Fig. \ref{fig:sub_dist}, CSRS shows superior performance while maintaining the proportion of high-confidence samples at a lower level. Fig. \ref{fig:entropy} also demonstrates a slower rate of entropy reduction compared to majority voting. All results indicate that our method results in high accuracy without suffering from the degeneracy typical of self-evolution, where a model becomes overly confident in a set of responses.

\textbf{Effect of Softened Frequency Reward.} Our softened frequency reward plays a significant role in CSRS. Crucially, as illustrated in Fig. \ref{fig:sub_correct}, the proportion of correct samples across all frequency ranges exhibits a consistent upward trend as training progresses. This empirically validates our theoretical derivation in Section \ref{math_proof}. 

\textbf{Effect of Visual Semantic Perturbation.} Finally, we evaluate the impact of visual perturbation as shown in Table \ref{tab:ablation_column_width}. Removing this perturbation leads to performance decay across all three benchmarks. This shows that visual perturbation prevents the model from relying on superficial visual features, compelling it to prioritize deep mathematical logic and thus enhancing the overall robustness of the reasoning process.

\textbf{Ablation of hyperparameters.} As shown in Table \ref{tab:exp_rollout}, the performance improves as the number of re-inference rollouts increases, reaching its peak at 5. While increasing rollouts initially facilitates the exploration of diverse information around anchored trajectories, exceeding this threshold leads to cumulative errors and introduces excessive computational overhead.

Regarding the retracing rate shown in Table \ref{tab:exp_retracing}, we observe a fluctuating yet upward trend, with performance peaking at 0.7. An excessively low retracing rate causes the search space to collapse back into the exponentially vast domain characteristic of traditional majority voting. Conversely, a high retracing rate implies the reasoning path is already functionally deterministic, offering negligible room for further refinement.

The sensitivity analysis of $\gamma$ and $\beta$ is presented in Table \ref{tab:sfr_sensitivity}. We observe that the model achieves optimal accuracy when $\gamma=0.2$ and $\beta=5.0$. The results demonstrate that small deviations result in only minor performance fluctuations, confirming the algorithmic robustness of our SFR against hyperparameter variations.

\textbf{Visualization of Maternal Answers Set and Re-inference Answer Set.} We use all-MiniLM-L6-v2 \citep{sentence_transformers_2024} to perform a semantic space analysis on the sets of maternal trajectories and re-inference trajectories during the training process. The corresponding t-SNE visualizations at steps 20, 60, and 100 are provided in Fig. \ref{fig:tsne}. It is observed that the semantic scope of the maternal trajectories progressively contracts into a relatively confined space. However, CSRS mitigates early-stage bias and maintains a substantially larger search space around anchor points,  facilitating the exploration of long-tail samples and alleviating the risk of model collapse.

\textbf{Other Model Family, Stability and Anchors Strategy.} To further validate the robustness of CSRS, we conducted additional experiments comparing stochastic versus fixed retracing anchors, evaluating performance across diverse model families, and assessing evolutionary stability. Detailed results for these analyses are provided in Appendix \ref{model_family}, \ref{convergence}, \ref{strategy}. The experimental evidence consistently demonstrates that our proposed strategies are optimal, exhibiting superior robustness and stability across various architectures and self-evolution cycles on geometric reasoning tasks.

\section{Conclusion}

This paper presents Continuous Softened Retracing reSampling (CSRS) designed to address the challenge of model collapse on math tasks in the unsupervised self-evolution of  MLLMs. CSRS introduces a retracing re-inference mechanism to effectively explore model’s coverage of the logical long-tail distribution samples. It implements softened frequency reward and visual perturbation dynamically calibrate the frequency variance across sampling distributions, steering the model away from over-fitting to high-frequency biases.
Empirical results demonstrate that CSRS consistently outperforms contemporary unsupervised methods across multiple multimodal benchmarks.

\section*{Limitations}

While CSRS demonstrates significant potential in the field of unsupervised MLLM self-evolution, we acknowledge several limitations that provide promising directions for future research:

Our current evaluation and methodology are primarily focused on mathematical reasoning tasks. It remains to be fully explored about broader, open-ended general scene understanding. Extending CSRS to diverse multimodal tasks is a key objective for our future work.

Looking ahead, we aim to further investigate the evolution of intrinsic reward mechanisms from static logical verification into dynamic, self-reflective evaluation, which would facilitate preciser guidance of reasoning quality in scenarios.


\bibliography{custom}

\appendix

\section{Appendix}
\label{sec:appendix}

\subsection{Proof of Section 3.1}
\label{Proof_of_pro}

To theoretically justify how \textbf{Continuous Softened Retracing reSampling (CSRS)} alleviates model collapse, we analyze the evolutionary dynamics of the policy under different reward mechanisms.

\begin{proposition}[Self-evolution Closed-form Solution]
In the GRPO optimization objective, adopting majority voting leads to distribution collapse, and its ideal closed-form solution \(P(x)\) and the reference distribution \(P_{\text{ref}}(x)\) satisfy the following relation:
\[
P(x) = \frac{P_{\text{ref}}(x) e^{A_n(x)/\beta}}{Z}
\]
where \(Z\) is the normalization constant.
\end{proposition}

\begin{proof}[Proof]
According to the Lagrange multiplier method, we construct the following functional:
Under the GRPO optimization objective, the ideal closed-form solution is given by:
\[
\max_{P} \, \mathbb{E}_{x \sim P}[A_n(x)]-\beta \, \mathbb{D}_{\text{KL}}\bigl(P \parallel P_{\text{ref}}\bigr)
\]
where \(P\) denotes the optimal distribution to be found. Applying the Lagrange multiplier method yields:
\[
\begin{aligned}
\mathcal{L}(P) = & \sum_x P(x)A_n(x) - \beta \sum_x P(x) \log \frac{P(x)}{P_{\text{ref}}(x)} \\
& + \lambda \Bigl( \sum_x P(x) - 1 \Bigr)
\end{aligned}
\]
with \(\lambda\) being the Lagrange multiplier. Differentiating with respect to \(P(x)\) gives:
\[
\frac{\partial \mathcal{L}}{\partial P(x)} = A_n(x)-\beta\Bigl(\log\frac{P(x)}{P_{\text{ref}}(x)}+1\Bigr)+\lambda = 0 .
\]
Solving this yields
\[
P(x)=P_{\text{ref}}(x)e^{\frac{A_n(x)}{\beta}+C}= \frac{P_{\text{ref}}(x)e^{\frac{A_n(x)}{\beta}}}{Z}.
\]
Consequently, for a MLLM that uses majority voting for self‑evolution, its gradient dynamics follow the equation above, where
\[
A_n(x)=\mathbf{1}(x\in M_n)-\rho_n,
\]
where \(M_n\) the set of majority votes and \(\rho_n\) the average proportion of the majority.  
When the majority voting mechanism is used, the value of \(A_n(x)\) at majority points is large, so that \(e^{A_n(x)/\beta}\) is dominated by \(A_n(x)\) (here \(\beta\) is a constant). Although the KL divergence term can be partly constrained by \(\beta\), the self evolution trend still collapses towards the majority region or areas of high probability density in the model’s own distribution. In the limit, this distribution collapse degenerates into an extreme, fixed mapping.
\end{proof}

\begin{proposition}[CSRS Can Relieve Model Collapse]
When a reinforcement learning algorithm uses GRPO, softening the frequency reward compared with the 0–1 reward of majority voting alleviates the distribution‑collapse phenomenon in LLM self‑evolution.
\end{proposition}

\begin{proof}[Proof]
We have \(A_n(x)=r(x)-\bar{r}\). For the softened‑frequency reward,
\[
A_{\text{sf}}(x)=P_n(x)-\bar{P_n}.
\]
Since
\[
\begin{aligned}
\mathbb{E}[\bar{r}] &= \mathbb{E}_{x \sim P_n}[r(x)] \\
&= \sum_{y \in P_n} P_n(y)r(y) \\
&= \sum_{y \in P_n} P_n(y) P_n(y) \\
&= \sum_{y \in P_n} P_n^2(y)
\end{aligned}
\]
it follows that
\[
A_{\text{sf}}(x)=P_n(x)-\sum_{y\in P_n} P_n^2(y).
\]
According to the gradient dynamic expression in \textbf{Proposition 1}, the update with the softened‑frequency reward reads
\[
P_{n+1}(x)=P_n(x)\,
   \exp\!\Bigl(\frac{\eta\bigl(P_n(x)-\sum_{y\in P_n}P_n^2(y)\bigr)}{Z}\Bigr).
\]

Let \(x_{\text{tail}}\) denote a long tail event in the self‑evolution.  
For majority voting, \(A_{\text{MV}}=0-\rho=-\rho\), which gives
\[
P_{n+1}(x)=P_n(x)e^{-\eta\rho_n}.
\]
For the softened frequency reward, writing \(P_n(x_{\text{tail}})=\epsilon\), we obtain
\[
P_{n+1}(x_{\text{tail}})=P_n(x_{\text{tail}})\,
   \frac{e^{\eta\epsilon}}{Z\,e^{\eta\bar{P_n}}}.
\]

Define a contrast operator: let \(x_1\) be majority samples and \(x_2\) long‑tail samples. The contrast at the \(n\)-th iteration is
\[
R_n = \frac{P_n(x_1)}{P_n(x_2)}.
\]
The contrast at the next generation is
\[
R_{n+1}=R_n\,
   \exp\!\bigl(\eta\,[r_n(x_1)-r_n(x_2)]\bigr).
\]
Define the contrast factor as
\[
G_n = \frac{R_{n+1}}{R_n}= \exp\!\bigl(\eta\,\Delta r_n\bigr),
\]
where
\[
\Delta r_n = 
\begin{cases} 
1,                     & \text{if } x \in \text{Distri(MV)}, \\[2pt]
\rho - \epsilon,       & \text{if } x \in \text{Distri(SF)}.
\end{cases}
\]
Because \(0<\rho-\epsilon<1\), we have \(G_{\text{sf}} < G_{\text{mv}}\). Hence the softened‑frequency reward slows down the rate of distribution collapse.
\end{proof}

\subsection{Implementation Details}

\subsubsection{Compute Resources}
We conduct our experiments using NVIDIA  A800-80G GPUs. The experimental time using 8 A800 for training Qwen2.5-VL-7B \citep{bai2025qwen2} on the Geometry3K \citep{wang2024measuring},GeoQA \citep{chen2021geoqa}, MMR1 \citep{leng2025mmr1} dataset using GRPO is around 36 hours.

\subsubsection{Entropy Change}

\begin{figure}[h]
    \centering
    \includegraphics[width=0.35\textwidth]{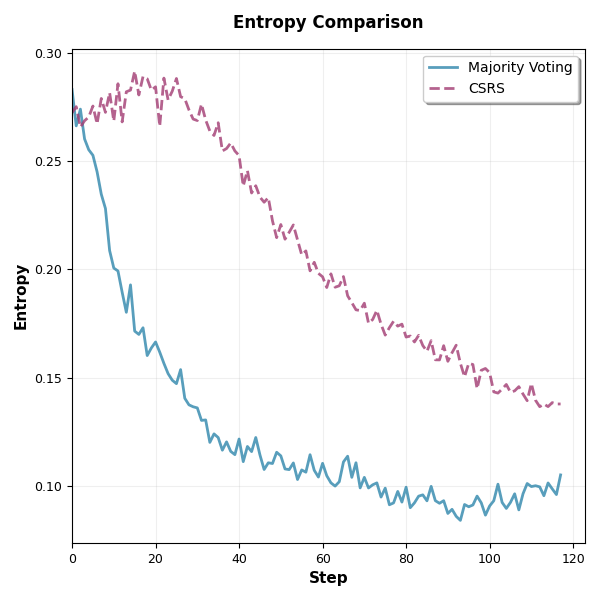}
    \caption{Entropy Change during training}
    \label{fig:entropy}
\end{figure}

We further compare the evolution of policy entropy between our proposed CSRS and the Majority Voting baseline (MM-UPT) throughout the training process, as illustrated in Fig. \ref{fig:entropy}. It is evident that the entropy in CSRS exhibits a slower and more gradual decay compared to MM-UPT. Unlike the binary reward structure of majority voting, CSRS maintains higher training stability and effectively mitigates model collapse. This gentler transition allows the model to have more access to long-tail samples, thereby facilitating the acquisition of more diverse and informative learning signals.

\subsubsection{Hyperparameters}
\label{config}

Our complete hyperparameters and system configurations are shown in Table \ref{tab:all_hyperparams}.

\subsubsection{Modules Ablation Studies Table}
Due to space constraints, the detailed results for the module ablation study are presented in Table \ref{tab:ablation_column_width}.

\begin{table}[t]
  \centering
  \small
  \caption{Ablation studies of different modules. The results demonstrate the incremental contribution of each component to the overall performance.}
  \label{tab:ablation_column_width}
  \begin{tabular*}{\columnwidth}{@{\extracolsep{\fill}}lcccc}
    \toprule
    \textbf{Module} & \textbf{M-Vis.} & \textbf{M-Ver.} & \textbf{M-Vist.} & \textbf{WeMath} \\ 
    \midrule
    MajorityVote              & 25.40 & 44.24 & 66.42 & 67.65 \\ % 保持 Part 1 的 Baseline 数据一致
    +SFR            & 26.52 & 45.10 & 67.15 & 69.42 \\ % SFR 提升明显
    +RRM            & 26.85 & 45.35 & 67.38 & 69.85 \\ % RRM 提升明显
    +VSP            & 25.55 & 44.35 & 66.55 & 67.82 \\ % VSP 提升较小
    +SFR+RRM        & 27.68 & 45.85 & 67.62 & 71.30 \\ % 组合效果接近最终值
    +SFR+VSP        & 26.70 & 45.22 & 67.28 & 69.65 \\ 
    +RRM+VSP        & 27.02 & 45.48 & 67.50 & 70.12 \\ 
    \midrule
    \textbf{CSRS}   & \textbf{27.97} & \textbf{46.01} & \textbf{67.81} & \textbf{71.77} \\ % 保持最终数据
    \bottomrule
  \end{tabular*}
\end{table}

\begin{table}[h]
\centering
\small
\caption{Complete Hyperparameter and System Configurations.}
\label{tab:all_hyperparams}
\begin{tabular*}{\columnwidth}{@{\extracolsep{\fill}}ll@{}}
\toprule
\textbf{Category} & \textbf{Hyperparameter (Value)} \\ \midrule

% --- 第一部分 ---
\rowcolor{lightgray} \multicolumn{2}{l}{\textbf{Data Configuration}} \\
Train Batch Size & 256 \\
Max Prompt Length & 1024 \\
Max Response Length & 1536 \\
Filter Overlong Prompts & True \\
Truncation Strategy & "error" \\
Image Key & "images" \\ \addlinespace[0.5em]

% --- 第二部分 ---
\rowcolor{lightgray} \multicolumn{2}{l}{\textbf{Model \& Optimization}} \\
Base Model & Qwen2.5-VL-7B-Instruct \\
Optimizer Learning Rate & 1e-6 \\
KL Loss Coefficient ($\beta$) & 0.01 \\
KL Loss Type & Low-Var KL \\
Entropy Coefficient & 0 \\
Gradient Checkpointing & Enabled \\
Remove Padding / Fused Kernels & Enabled \\ \addlinespace[0.5em]

% --- 第三部分 ---
\rowcolor{lightgray} \multicolumn{2}{l}{\textbf{PPO / GRPO Settings}} \\
Algorithm & GRPO \\
PPO Mini Batch Size & 64 \\
PPO Micro Batch Size (per GPU) & 8 \\
Rollout Group Size ($G$) & 8 \\
Ref Log-Prob Micro Batch (GPU) & 4 \\
Param Offload (Ref Model) & True \\ \addlinespace[0.5em]

% --- 第四部分 ---
\rowcolor{lightgray} \multicolumn{2}{l}{\textbf{Sampling \& Retracing (CSRS)}} \\
Rollout Engine & vLLM \\
Sampling Group Size ($n$) & 8 \\
Re-inference Number & 5 \\
Retracing Rate & 0.7 \\
GPU Memory Utilization & 0.6 \\ \addlinespace[0.5em]

% --- 第五部分 ---
\rowcolor{lightgray} \multicolumn{2}{l}{\textbf{Trainer Settings}} \\
Total Epochs & 15 \\
GPUs per Node & 8 \\
Number of Nodes & 1 \\
Test Frequency & 5 epochs \\
Save Frequency & 20 steps \\ \bottomrule
\end{tabular*}
\end{table}

\subsection{Model Families Experiments}
\label{model_family}
In multimodal self-evolution tasks, the ability of a framework to generalize across different model families is a key indicator of its algorithmic robustness. We expand our experiments to include two additional prominent MLLM families: InternVL3-8B and ThinkLite-VL-7B. As shown in the Table \ref{tab:comprehensive_results} below, the results prove that CSRS consistently delivers significant performance gains across different model architectures and parameter scales.

\begin{table*}[t]
  \centering
  \footnotesize % 1. 缩小字体，显得更精致专业
  \caption{Comprehensive performance comparison and ablation studies after integrating CSRS.}
  \label{tab:comprehensive_results}
  \renewcommand{\arraystretch}{1.1} % 稍微增加行高，防止文字太挤
  \setlength{\tabcolsep}{12pt}      % 2. 增加列间距，利用双栏宽度
  
  \begin{tabular}{lcccccc} % 增加到 6 列
    \toprule
    \textbf{Model / Datasets} & \textbf{MathVision} & \textbf{MathVerse} & \textbf{MathVista} & \textbf{WeMath} & \textbf{Avg.} \\ 
    \midrule
    
    % --- 第一组：不同基座模型 ---
    \multicolumn{6}{l}{\textit{Part 1: Performance of  Base Models and Trained Models in \textbf{Geometry3k}}} \\ \midrule
    Qwen2.5-VL-7B(Vanilla)    & 25.40 & 44.24 & 66.42 & 67.65 & 50.93 \\
    +CSRS            & 27.97 & 46.01 & 67.81 & 71.77 & \textbf{53.39} \\ \cmidrule(lr){1-6}
    InternVL3-8B(Vanilla)     & 29.30 & 45.12 & 71.60 & 73.25 & 54.82 \\
    +CSRS            & 30.58 & 47.23 & 72.98 & 75.15 & \textbf{56.49} \\ \cmidrule(lr){1-6}
    ThinkLite-VL-7B(Vanilla)  & 25.92 & 46.01 & 68.12 & 67.12 & 51.79 \\
    +CSRS            & 26.98 & 48.25 & 69.76 & 69.32 & \textbf{53.58} \\ \midrule
    
    % --- 第二组：Rollout Size ---
    \multicolumn{6}{l}{\textit{Part 2: Performance of  Base Models and Trained Models in \textbf{GeoQA}}} \\ \midrule
    Qwen2.5-VL-7B+CSRS           & 28.95 & 45.82 & 68.25 & 69.32 & 53.09 \\ 
    InternVL3-8B+CSRS            & 31.82 & 47.50 & 73.40 & 75.80 & \textbf{57.13} \\ 
    ThinkLite-VL-7B+CSRS            & 28.15 & 48.35 & 70.22 & 69.45 & \textbf{54.04} \\ \midrule
    
    % --- 第三组：Retracing Rate ---
    \multicolumn{6}{l}{\textit{Part 3: Performance of  Base Models and Trained Models in \textbf{MMR1}}} \\ \midrule
    Qwen2.5-VL-7B+CSRS            & 27.86 & 45.89 & 67.81 & 70.53 & 53.05 \\ 
    InternVL3-8B+CSRS            & 30.95 & 46.42 & 72.15 & 75.20 & \textbf{56.18} \\ 
    ThinkLite-VL-7B+CSRS            & 27.92 & 47.95 & 69.80 & 70.15 & \textbf{53.96} \\ 
    \bottomrule
  \end{tabular}
\end{table*}

\subsection{Convergence and Stablity Analysis}
\label{convergence}

Since the ground-truth labels are inaccessible in self-evolution tasks, the model faces the risk of converging to a local optimum. Consequently, we conducted additional long-term evolution experiments to verify whether performance degrades as training duration increases.  We evaluated the model every 5 training steps on Geometry3K. As illustrated in Table \ref{tab:training_steps}, CSRS maintains steady performance without the significant degradation or label contamination. This proves our framework's robustness in maintaining label purity and training stability during long-term self-evolution.

% --- Table 7: Training Steps ---
\begin{table}[t]
  \centering
  \small % 统一字号
  \caption{Performance across different training steps.}
  \label{tab:training_steps}
  \setlength{\tabcolsep}{4.5pt} % 针对 6 列稍微收紧间距
  \begin{tabular}{lccccc}
    \toprule
    \textbf{Steps} & \textbf{M-Vis.} & \textbf{M-Ver.} & \textbf{M-Vist.} & \textbf{WeMath} & \textbf{Avg.} \\ 
    \midrule
    120 & 27.97 & 46.01 & 67.81 & 71.77 & 53.39 \\
    125 & 27.52 & 47.24 & 68.06 & 71.64 & 53.62 \\
    130 & 28.03 & 46.18 & 67.25 & 70.93 & 53.10 \\
    135 & 28.01 & 46.75 & 67.99 & 71.62 & 53.59 \\ 
    140 & 28.25 & 47.34 & 67.63 & 70.89 & 53.53 \\ 
    \bottomrule
  \end{tabular}
\end{table}

\subsection{Random Re-inference VS Fixed Re-inference}
\label{strategy}
As our RRM uses a fixed-anchor re-inference approach, we conduct comparative experiments on the Geometry3K dataset to verify the effectiveness of this mechanism, comparing fixed anchors with random anchor re-inference. For the stochastic retracing variant, the retracing rate $\omega$ is sampled from a uniform distribution, such that $\omega \sim \mathcal{U}(0, 1)$ at each step.

The results in Table \ref{tab:sampling_ablation} and Table \ref{tab:efficiency} indicate that while accuracy remains similar, the stochastic retracing strategy results in a 1.13$\times$ computational overhead compared to our fixed-rate approach. This efficiency gap stems from two main factors. First, when the retracing rate $\omega$ is randomly chosen from the range $[0,1]$ and happens to be small, the model is forced to re-generate the entire problem description, which wastes time and computing resources. Second, since GRPO requires all samples in a group to finish before the next step, the large variation in response lengths caused by a random $\omega$ creates a bottleneck: the entire process must wait for the longest redundant response to complete. Furthermore, in a self-evolution setting without explicit ground-truth labels, a stable training strategy is crucial to prevent the model from falling into local optima or training collapse, allowing it to reach proficiency more rapidly. In contrast, CSRS uses a fixed anchor and significantly improves training speed, which keeps the generation length consistent and stable. 

% --- Table 8: Ablation ---
\begin{table}[t]
  \centering
  \small % 保持与 Table 7/9 统一的字号
  \caption{Ablation of sampling strategies.}
  \label{tab:sampling_ablation}
  
  % @{\hspace{2pt}} 强制减小第一列和第二列之间的物理距离
  % 你可以根据需要把 2pt 改为 0pt 甚至更小
  \begin{tabular}{l @{\hspace{5pt}} cccc}
    \toprule
    \textbf{Method} & \textbf{M-Vis.} & \textbf{M-Ver.} & \textbf{M-Vist.} & \textbf{WeMath} \\ 
    \midrule
    CSRS (Ran.) & 27.63 & 46.35 & 67.52 & 72.01 \\
    CSRS (Fix.) & 27.97 & 46.01 & 67.81 & 71.77 \\ 

    \bottomrule
  \end{tabular}
\end{table}

\begin{table}[t]
    \centering
    \small
    \caption{Efficiency and Accuracy Comparison.}
    \label{tab:efficiency}
    % 使用 tabular* 并设定宽度为 \columnwidth
    \begin{tabular*}{\columnwidth}{@{\extracolsep{\fill}}lcc@{}}
        \toprule
        Method & Relative Training Time & Acc. \\
        \midrule
        CSRS Fixed & 1.0$\times$ & 27.97 \\
        CSRS Random & 1.13$\times$ & 27.63 \\
        \bottomrule
    \end{tabular*}
\end{table}

\subsection{Case Study}
\label{case_study}

In this section, we demonstrate how CSRS operates on mathematical problems during the training process. As illustrated in following boxes,  within the initial set of maternal trajectories generated by the model, the majority of reasoning paths exhibit typical errors such as relying on superficial mathematical cues, convoluted logical reasoning, and incorrect initial directions. In this instance, only a single correct reasoning path exists. Under a conventional majority voting scheme, the model would inevitably evolve toward its own distributional bias (e.g., the incorrect answer "47"). However, with the application of CSRS, a higher frequency of correct reasoning paths emerges within the re-inference set. Furthermore, correct paths with low frequency are no longer penalized with a zero reward, thereby preserving valuable learning signals.

\begin{PromptBox}{1. Maternal Rollouts for Problem}
\textbf{Maternal Response 1}: To find angle $BDA$, we note that the inscribed angle $BDC$ intercepts arc $CB$ $(94^\circ)$, so $\angle BDC = 47^\circ$. Since $AB \parallel DC$, $\angle BDA = \angle BDC$. Therefore, $\angle BDA = 47^\circ$.

\textbf{Maternal Response 2}: Using cyclic quadrilateral properties, $\angle BCA$ intercepts arc $CB$, so $\angle BCA = 47^\circ$. Opposite angles in a cyclic quadrilateral sum to $180^\circ$, so $\angle BDA = 180^\circ - 47^\circ = 133^\circ$.

\textbf{Maternal Response 3}: $\angle BDC$ intercepts arc $CB$, giving $\angle BDC = 47^\circ$. Since $AB \parallel DC$, $\angle BDA = \angle BDC$. Therefore, $\angle BDA = 47^\circ$.

\textbf{Maternal Response 4}: $\angle BDC = \frac{1}{2} \cdot \widehat{CB} = 47^\circ$. Due to the parallel lines, $\angle BDA = \angle BDC$. Thus, $\angle BDA = 47^\circ$.

\textbf{Maternal Response 5}: $\angle BDA$ is an inscribed angle intercepting the same arc as central angle $BZA$. So $\angle BDA = \frac{1}{2} \cdot \angle BZA = \frac{1}{2} \times 104^\circ = 52^\circ$.

\textbf{Maternal Response 6}: Since $AB \parallel DC$, $\angle BAC = \angle BDC = \frac{1}{2} \cdot \widehat{CB} = 47^\circ$. And $\angle BDA = \angle BAC$, so $\angle BDA = 47^\circ$.

\textbf{Maternal Response 7}: In the cyclic quadrilateral, $\angle BCA$ intercepts arc $CB$, so $\angle BCA = 47^\circ$. The opposite angle sum gives $\angle BDA = 180^\circ - 47^\circ = 133^\circ$.

\textbf{Maternal Response 8}: We find arc $BA$ from central angle $BZA$: $\widehat{BA} = 2 \times 104^\circ = 208^\circ$. $\angle C$ intercepts arc $CB$: $\angle C = 47^\circ$. Using opposite angle sum in the cyclic quadrilateral, $\angle A = 133^\circ$. Due to parallel lines, $\angle BDA = \angle C = 47^\circ$.
\end{PromptBox}

\begin{PromptBox}{2. Cut rollouts for Problem}
\textbf{Re-inference Response 9}: The central angle $BZA$ is given as $104^\circ$. The inscribed angle $BDA$ subtends the same arc $BA$. By the inscribed angle theorem, $\angle BDA = \frac{1}{2} \cdot \angle BZA = 52^\circ$. The parallel lines condition is not needed for this calculation.

\textbf{Re-inference Response 15}: The reasoning is as follows: The central angle $BZA$ is $104^\circ$. Since the inscribed angle $BDA$ subtends the same arc $BA$, its measure is half of the central angle, which is $52^\circ$. The information about parallel lines is not used.

\textbf{Re-inference Response 18}: $\angle BDA$ subtends arc $BA$, and central angle $BZA$ subtends the same arc. By the inscribed angle theorem, $\angle BDA = \frac{1}{2} \times 104^\circ = 52^\circ$.

\textbf{Re-inference Response 35}: We use the inscribed angle theorem: The measure of an inscribed angle is half the measure of the central angle that subtends the same arc. Therefore, $\angle BDA = \frac{1}{2} \times 104^\circ = 52^\circ$.

\textbf{Re-inference Response 39}: The measure of $\angle BDA$ is determined by the central angle $BZA$ ($104^\circ$) since they subtend the same arc $BA$. Applying the inscribed angle theorem gives $\angle BDA = 52^\circ$. The parallel condition is not necessary for this solution.
\end{PromptBox}

The box3 and box4 illustrate the specific formatting of prompts and responses are used during the training phase. Notably, the entire reasoning trajectory is encapsulated within <think></think> tags to facilitate structured cognitive processing, while the final conclusion is presented within a \textbackslash boxed\{\} environment for automated extraction.

\begin{PromptBox}{3. Maternal Rollout Example}
    \textbf{prompt}:"system You are a helpful assistant. user For the pair of similar figures, use the given areas to find $x$. You FIRST think about the reasoning process as an internal monologue and then provide the final answer. The reasoning process MUST BE enclosed within <think> </think> tags. The final answer MUST BE put in \textbackslash \textbackslash boxed{}.assistant"
    
    \textbf{output}:"<think> To find the value of  x , we need to use the properties of similar figures. The ratio of the areas of two similar figures is equal to the square of the ratio of their corresponding side lengths.1. Let the ratio of the corresponding side lengths be  k.  2. The area of the smaller figure is  50  $in^2$.  3. The area of the larger figure is  72  $in^2.4$. The ratio of the areas is  \textbackslash frac\{50\}\{72\} .5. Since the areas are proportional to the square of the side lengths, we have  $k^2$ = \textbackslash frac\{50\}\{72\} . 6. Solving for  k, we get  k = \textbackslash sqrt\{\textbackslash frac\{50\}\{72\}\} = \textbackslash sqrt\{\textbackslash frac\{25\}\{36\}\} = \textbackslash frac\{5\}\{6\} . 7. The side length  x  corresponds to the smaller figure, and the side length of the larger figure is 6 inches. 8. Therefore,  x = 6  times \textbackslash frac\{5\}\{6\} = 5 . So, the value of  x  is 5  inches.</think> \textbackslash \textbackslash boxed\{5\}"
\end{PromptBox}

\begin{PromptBox}{4. Retracing Rollout Example}
    \textbf{prompt}:"system You are a helpful assistant. user Find $x$. Assume that segments that appear to be tangent are tangent. You FIRST think about the reasoning process as an internal monologue and then provide the final answer. The reasoning process MUST BE enclosed within <think> </think> tags. The final answer MUST BE put in \textbackslash \textbackslash boxed\{\}. assistant <think> To find the value of  x, we need to use the property of tangent segments to a circle.assistant"
    
    \textbf{output}:"<think>The tangent segments from an external point to a circle are equal in length.In the given diagram, the segment that appears to be tangent from the external point to the circle is 14, and the segment that appears to be tangent to the circle from the same external point is  x .Since the segments are tangent to the circle from the same external point, they are equal in length. Therefore, we can set up the equation: x = 14 So, the value of  x  is 14.</think> \textbackslash\textbackslash boxed\{14\}"
\end{PromptBox}

As a representative error analysis, box5 illustrates how logical bias propagates even after re-inference: In our observations, when re-inference is initiated from a retracing anchor point, the resulting trajectories exhibit two distinct behaviors: logical persistence and strategic reflection. If the initial reasoning direction is flawed, persistent trajectories often lead to consistent but incorrect conclusions. Conversely, the reflective paths demonstrate a self-correction capability, where the model deviates from the original erroneous logic to discover the valid solution.

\begin{PromptBox}{5. Bad Cases}
    \textbf{Concatenated Prompt}: You are a helpful assistant. \\
    user. Find $JL$, if $KP = 4$ and $PM = 7$. You FIRST think about the reasoning process as an internal monologue and then provide the final answer. The reasoning process MUST BE enclosed within <think> </think> tags. The final answer MUST BE put in \textbackslash boxed\{\}. In the diagram, we have a kite shape with the diagonals as the main focus. The diagonals of a kite are perpendicular to each other and bisect each other at right angles. Here, diagonal $KS$ is perpendicular to diagonal $JL$ at point $P$.

    \vspace{0.5em}
    \textbf{Correct Response (Maternal)}: \\
    The diagonals of the kite/rhombus intersect at point $P$ and bisect each other. Since $KP = 4$ and $PM = 7$, the total length of diagonal $JL$ is the sum of these segments: $JL = KP + PM = 4 + 7 = 11$. $\boxed{11}$

    \vspace{0.5em}
    \textbf{Incorrect Response 1}: \\
    Treating the shape as a kite where the diagonals are perpendicular and bisect each other, one might incorrectly assume $PM$ is half of $JL$, leading to $JL = 2 \times PM = 14$. $\boxed{14}$

    \vspace{0.5em}
    \textbf{Incorrect Response 2}: \\
    Assuming the figure is a rhombus, the side length $JL$ might be calculated as the hypotenuse of the right triangle formed by the half-diagonals: $JL = \sqrt{KP^2 + PM^2} = \sqrt{65}$. $\boxed{\sqrt{65}}$

    \vspace{0.5em}
    \textbf{Incorrect Response 3}: \\
    Interpreting the figure as a parallelogram where diagonals bisect each other, one might incorrectly sum the segments and double them: $JL = 2 \times (KP + PM) = 22$. $\boxed{22}$

    \vspace{0.5em}
    \textbf{Incorrect Response 4}: \\
    Assuming symmetry incorrectly leads to $JL = KP = 4$, then the full diagonal length is doubled: $JL = 2 \times 4 = 8$. $\boxed{8}$
\end{PromptBox}

\begin{figure*}[t] % 星号使其横跨双栏，[t] 通常建议放在页面顶部
    \centering
    \includegraphics[width=0.95\textwidth]{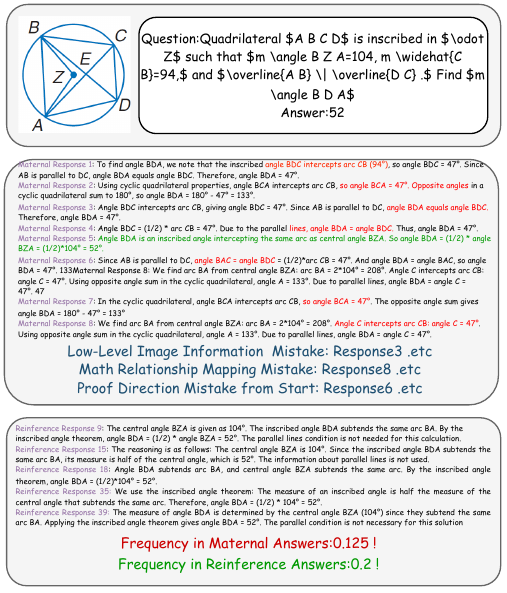} % 建议调大宽度，比如 0.8\textwidth
    \caption{Case studies in geometric reasoning.}
    \label{fig:case_pic}
\end{figure*}

\end{document}